\documentclass{article}

     \PassOptionsToPackage{numbers, compress}{natbib}

      \usepackage[preprint]{neurips_2024}

\usepackage[utf8]{inputenc} 
\usepackage[T1]{fontenc}    
\usepackage{hyperref}       
\usepackage{url}            
\usepackage{booktabs}       
\usepackage{amsfonts}       
\usepackage{amsmath}
\usepackage{amsthm}
\usepackage{nicefrac}       
\usepackage{microtype}      
\usepackage{xcolor}         

\newtheorem{thm}{Theorem}
\newtheorem{corr}{Corrolary}
\newtheorem{defn}{Definition}

\title{All Random Features Representations are Equivalent}

\author{%
  Luke Sernau \\
  Google DeepMind \\
  \texttt{sernau@google.com} \\
  \And
  Silvano Bonacina\\
  Google DeepMind \\
  \texttt{sibonaci@google.com} \\
  \And
  Rif A. Saurous\\
  Google Research \\
  \texttt{rif@google.com} \\
}

\begin{document}

\maketitle

\begin{abstract}
  Random features are a powerful technique for rewriting positive-definite kernels as linear products. They bring linear tools to bear in important nonlinear domains like KNNs and attention. Unfortunately, practical implementations require approximating an expectation, usually via sampling. This has led to the development of increasingly elaborate representations with ever lower sample error.
  
  We resolve this arms race by deriving an optimal sampling policy. Under this policy all random features representations have the same approximation error, which we show is the lowest possible. This means that we are free to choose whatever representation we please, provided we sample optimally.
\end{abstract}

\section{Introduction}

Kernel methods have a long and illustrious history in machine learning \cite{kernel_methods}. Nowadays one of their most widespread applications is the softmax kernel inside the attention layers of a transformer. Unfortunately, kernel methods scale linearly in the size of the dataset, which is impractical for very large datasets.

In their seminal paper, \citet{random_features} show that this linear cost can be avoided by replacing the kernel with a randomized approximation that is correct in expectation. They did this by means of a \emph{random feature representation}.

\begin{defn}
Given a positive-definite kernel $K: X\times X\rightarrow\mathbb{R}$, a \emph{random feature representation} of $K$ is a symmetric function $\phi: X\times X\rightarrow\mathbb{R}$ together with a distribution \(\Omega\) (most commonly the normal distribution) such that
\begin{equation}\label{random_features}
K\left(x_1,x_2\right) = \mathop{\mathbb{E}}_{\omega \sim \Omega}\left[\phi\left(x_1, \omega \right)\phi\left(x_2, \omega \right)\right].
\end{equation}
\end{defn}

To see why this is useful, consider a simple kernel estimator over some dataset \(\{(x_1, y_1), (x_2, y_2), ..., (x_n, y_n)\}\).

\begin{equation*}
KE\left(x\right) = \sum^n_{i=1}K\left(x,x_i\right)y_i.
\end{equation*}

As written, this kernel estimator requires linear time to evaluate, since for every new \(x\) it's necessary to compute a dot product with every element in the dataset. But if the kernel admits a random features representation, we may write
\begin{align*}
KE\left(x\right) &= \sum^n_{i=1}\mathop{\mathbb{E}}_{\omega \sim \Omega}\left[\phi(x, \omega)\phi(x_i, \omega)\right]y_i = \mathop{\mathbb{E}}_{\omega \sim \Omega}\left[\phi(x, \omega)\sum^n_{i=1}\phi(x_i, \omega)y_i\right].
\end{align*}

This expectation can be approximated via sampling,
\begin{equation*}
KE\left(x\right) \approx \frac{1}{k}\sum^k_{j=1}\phi(x, \omega_j)\sum^n_{i=1}\phi(x_i, \omega_j)y_i,
\end{equation*}
where \(\omega_1, \omega_2, \dots, \omega_k\) are samples drawn from \(\Omega\). As long as \(k\) is less than \(n\), we can save time by precomputing the sum on the right. The principle challenge is finding an estimator that 

\citet{favor} were the first to apply this idea to attention. They did so by finding a choice of \(\phi\) that was numerically stable and had low sample variance. This kicked off a flurry of research (\cite{chefs} \cite{favor_sharp} \cite{simplex}) into better choices of \(\phi\) with ever lower sample variance.

We provide a potential conclusion to this competition. Rather than optimizing \(\phi\), we simply chose our sampling strategy carefully. By approximating the expectation using an optimal importance sampling strategy, we are able to achieve a sample error that is as low as possible. 

Remarkably, this sample variance \emph{does not depend on \(\phi\)}, meaning that when sampled optimally, \emph{all} random features schemes have the same variance.

Ours is not the first scheme to consider importance sampling, but in contrast to others that sample based on kernel polarization heuristics \cite{data_dependent} or quadrature rules \cite{quadrature}, we directly solve the variance minimization problem. To our knowledge we are the first to obtain global optima that are independent of the choice of \(\phi\).

\section{Importance sampling}

Importance sampling \cite{importance_sampling} is a well known technique for reducing variance when estimating an expectation via sampling. It's based on the observation that some samples have a larger effect on the expectation than others. Instead of naively sampling from the given distribution, we sample from some new distribution \(\Psi\), rescaling the samples such that the correct final estimate is preserved.

In our case, we wish to estimate \eqref{random_features} via sampling. Let \(\Psi\) be some new distribution with the same support as \(\Omega\), and suppose \(\Psi\) and \(\Omega\) have densities \(p_\Psi\) and \(p_\Omega\), respectively. Observe that
\begin{equation}\label{importance_sampling}
K\left(x_1,x_2\right) = \mathop{\mathbb{E}}_{\omega \sim \Omega}\left[\phi\left(x_1, \omega \right)\phi\left(x_2, \omega \right)\right] = \mathop{\mathbb{E}}_{\omega \sim \Psi}\left[\frac{p_\Omega\left(\omega\right)}{p_\Psi\left(\omega\right)}\phi\left(x_1, \omega \right)\phi\left(x_2, \omega \right)\right].
\end{equation}

In other words, we are free to sample from \(\Psi\) without changing the value of the expectation. On the other hand, the sample variance, given by
\begin{equation*}
\mathop{\text{Var}}_{\omega \sim \Psi}\left[\frac{p_\Omega\left(\omega\right)}{p_\Psi\left(\omega\right)}\phi\left(x_1, \omega \right)\phi\left(x_2, \omega \right)\right],
\end{equation*}
does depend on \(\Psi\). Notice that the mean squared error of our sampling procedure is precisely this variance. All we need to do to minimize the per-sample error is find a choice of \(\Psi\) to make this quantity small.

\section{Optimal sampling}

Choosing an optimal \(\Psi\) is an intrinsically data-driven question, as the variance depends on the input. We'd like to minimize the expected sample variance over all inputs, as captured in the following definition.
\begin{defn}
Suppose \(x_1\) and \(x_2\) are sampled from \(\mathcal{X}_1\) and \(\mathcal{X}_2\), respectively. If we approximate the expression in \eqref{importance_sampling} by sampling from \(\Psi\), the \emph{expected sample variance} \(\mathcal{V}_\Psi\) over all \(x_1\) and \(x_2\) is given by
\begin{equation*}
\mathcal{V}_\Psi = \mathop{\mathbb{E}}_{x_1 \sim \mathcal{X}_1}\mathop{\mathbb{E}}_{x_2 \sim \mathcal{X}_2}\mathop{\text{Var}}_{\omega \sim \Psi}\left[\frac{p_\Omega\left(\omega\right)}{p_\Psi\left(\omega\right)}\phi\left(x_1, \omega \right)\phi\left(x_2, \omega \right)\right]
\end{equation*}
\end{defn}

The \(\Psi\) that minimizes \(\mathcal{V}_\Psi\) can be found analytically.

\begin{thm}\label{optimal}
Let \(K\) be a positive-definite kernel such that
\begin{equation*}
K\left(x_1,x_2\right) = \mathop{\mathbb{E}}_{\omega \sim \Psi}\left[\frac{p_\Omega\left(\omega\right)}{p_\Psi\left(\omega\right)}\phi\left(x_1, \omega \right)\phi\left(x_2, \omega \right)\right]
\end{equation*}
for some \(\phi\) and \(\Omega\), and every \(\Psi\) with the same support as \(\Omega\). Suppose \(x_1\) and \(x_2\) are sampled from \(\mathcal{X}_1\) and \(\mathcal{X}_2\), respectively. Then the expected sample variance \(\mathcal{V}_\Psi\) over all \(x_1\) and \(x_2\) will be minimized when
\begin{equation*}
p_\Psi(w) \propto p_\Omega(w) q_\phi\left(\omega\right),
\end{equation*}
where
\begin{equation*}
    q_\phi\left(\omega\right)=\sqrt{\mathop{\mathbb{E}}_{x_1 \sim \mathcal{X}_1}\left[\phi\left(x_1, \omega \right)^2\right]\mathop{\mathbb{E}}_{x_2 \sim \mathcal{X}_2}\left[\phi\left(x_2, \omega \right)^2\right]}.
\end{equation*}
The resulting optimal variance will be
\begin{equation*}
\mathcal{\hat{V}}_\Psi = \left(\mathop{\mathbb{E}}_{\omega \sim \Omega}q_\phi\left(\omega\right)\right)^2 - \mathop{\mathbb{E}}_{x_1 \sim \mathcal{X}_1}\mathop{\mathbb{E}}_{x_2 \sim \mathcal{X}_2}K\left(x_1,x_2\right)^2.
\end{equation*}
\end{thm}
\begin{proof}
Recall that variance can be written as
\begin{align*}
\mathop{\text{Var}}_{\omega \sim \Psi}&\left[\frac{p_\Omega\left(\omega\right)}{p_\Psi\left(\omega\right)}\phi\left(x_1, \omega \right)\phi\left(x_2, \omega \right)\right]=\\
&\mathop{\mathbb{E}}_{\omega \sim \Psi}\left[\frac{p_\Omega\left(\omega\right)^2}{p_\Psi\left(\omega\right)^2}\phi\left(x_1, \omega \right)^2\phi\left(x_2, \omega \right)^2\right] - \mathop{\mathbb{E}}_{\omega \sim \Psi}\left[\frac{p_\Omega\left(\omega\right)}{p_\Psi\left(\omega\right)}\phi\left(x_1, \omega \right)\phi\left(x_2, \omega \right)\right]^2.
\end{align*}
The second term is equal to \(K\left(x_1, x_2\right)^2\), meaning it does not depend on \(\Psi\). We turn our attention to the first term. The expected value of the first term is

\begin{align}\label{first_term_of_V}
\mathop{\mathbb{E}}_{x_1 \sim \mathcal{X}_1}\mathop{\mathbb{E}}_{x_2 \sim \mathcal{X}_2}\mathop{\mathbb{E}}_{\omega \sim \Psi}&\left[\frac{p_\Omega\left(\omega\right)^2}{p_\Psi\left(\omega\right)^2}\phi\left(x_1, \omega \right)^2\phi\left(x_2, \omega \right)^2\right]\nonumber\\
=&\mathop{\mathbb{E}}_{x_1 \sim \mathcal{X}_1}\mathop{\mathbb{E}}_{x_2 \sim \mathcal{X}_2} \int_{X}\frac{p_\Omega\left(\omega\right)^2}{p_\Psi\left(\omega\right)}\phi\left(x_1, \omega \right)^2\phi\left(x_2, \omega \right)^2 d\omega\nonumber\\
=&\int_{X}\frac{p_\Omega\left(\omega\right)^2}{p_\Psi\left(\omega\right)}q_\phi\left(\omega\right)^2 d\omega
\end{align}
We wish to find the \(p_\Psi\) that minimizes this expression, subject to the constraint that \(\int_X p_\Psi\left(\omega\right)d\omega = 1\). We'll solve this using Lagrange optimization. The Lagrangian is
\begin{equation*}
\mathcal{L} = \frac{p_\Omega\left(\omega\right)^2}{p_\Psi\left(\omega\right)}q_\phi\left(\omega\right)^2 + \lambda p_\Psi\left(\omega\right).
\end{equation*}
where \(\lambda\) is our Lagrange parameter. Differentiating with respect to \(p_\Psi\left(\omega\right)\),
\begin{equation*}
\frac{\partial \mathcal{L}}{\partial p_\Psi} = -\frac{p_\Omega\left(\omega\right)^2}{p_\Psi\left(\omega\right)^2}q_\phi\left(\omega\right)^2 + \lambda.
\end{equation*}

Setting this equal to zero and solving, we find that

\begin{equation*}
p_\Psi(\omega) = \frac{1}{\sqrt{\lambda}} p_\Omega(\omega)q_\phi\left(\omega\right)
\end{equation*}
as desired. The constraint that \(\int_X p_\Psi\left(\omega\right)d\omega = 1\) forces \(\sqrt{\lambda} = \int_X p_\Omega\left(\omega\right)q_\phi\left(\omega\right)d\omega 
\). By plugging this into \eqref{first_term_of_V}, we obtain the desired variance.
\end{proof}
This result tells us how to minimize the variance for a particular choice of \(\phi\). What is less clear in this form is that it also tells us something about optimizing over every \(\phi\). We will explore this in the next section.

\section{The choice of feature representation does not matter}
As written, Theorem \ref{optimal} appears to be a statement about how to get the most out of any particular choice of \(\phi\). While this is useful, the real insight is what it can tell us about every possible \(\phi\). We are ready now to begin exploring our main point: the optimal sample variance is often entirely independent of \(\phi\).

Our first step in this direction is a tight upper bound on the sampling variance, and therefore the error, that can be stated without reference to \(\phi\).
\begin{corr}\label{bound}
The optimal sample variance achieved in Theorem \ref{optimal} is upper bounded by
\begin{equation*}
\mathcal{\hat{V}}_\Psi \leq \mathop{\mathbb{E}}_{x_1 \sim \mathcal{X}_1}\mathop{\mathbb{E}}_{x_2 \sim \mathcal{X}_2}\left[K\left(x_1,x_1\right)K\left(x_2,x_2\right) - K\left(x_1,x_2\right)^2\right].
\end{equation*}
\end{corr}

This follows from Theorem \ref{optimal} by applying the Cauchy-Schwarz inequality to the first term in \(\mathcal{\hat{V}}\). Specifically,
\begin{align*}
\left(\mathop{\mathbb{E}}_{\omega\sim \Omega} q_\phi \left(\omega\right)\right)^2 =& \left(\mathop{\mathbb{E}}_{\omega\sim \Omega} \left[\sqrt{ \mathop{\mathbb{E}}_{x_1 \sim \mathcal{X}_1}\left[\phi\left(x_1, \omega\right)^2\right]} \cdot \sqrt{ \mathop{\mathbb{E}}_{x_2 \sim \mathcal{X}_2}\left[\phi\left(x_2, \omega\right)^2\right]}\right]\right)^2\\
\leq& \left(\mathop{\mathbb{E}}_{\omega\sim \Omega} \mathop{\mathbb{E}}_{x_1 \sim \mathcal{X}_1}\left[\phi\left(x_1, \omega\right)^2\right]\right)\left(\mathop{\mathbb{E}}_{\omega\sim \Omega} \mathop{\mathbb{E}}_{x_2 \sim \mathcal{X}_2}\left[\phi\left(x_2, \omega\right)^2\right]\right)\\
=&\mathop{\mathbb{E}}_{x_1 \sim \mathcal{X}_1}\mathop{\mathbb{E}}_{x_2 \sim \mathcal{X}_2}K\left(x_1, x_1\right)K\left(x_2, x_2\right).
\end{align*}

Written this way, our bound is itself suggestively similar to the Cauchy-Schwarz inequality. Indeed, since positive-definite kernels obey the Cauchy-Schwarz inequality, we know that this expression is positive as expected. More interestingly, since this expression does not depend on \(\phi\), this bound is the same regardless of which representation we're using.

In the (relatively common) case where \(\mathcal{X}_1 = \mathcal{X}_2\), we can go even further.
\begin{corr}\label{tight}
If \(\mathcal{X}_1 = \mathcal{X}_2\) then the bound in Corollary \ref{bound} holds with equality.
\end{corr}
This can be seen from Theorem \ref{optimal}, since in this case \(q_\phi\left(\omega\right) = \mathop{\mathbb{E}}_{x_1 \sim \mathcal{X}_1}\left[\phi\left(x_1, \omega \right)^2\right]\). By exchanging the expectations it follows that \(\mathop{\mathbb{E}}_{\omega \sim \Omega}q_\phi\left(\omega\right)=\mathop{\mathbb{E}}_{x_1 \sim \mathcal{X}_1}\left[K\left(x_1,x_1\right)\right]\).

The significance of this result is that we are able to write down the exact sample variance without any reference to \(\phi\). This is not just a statement about the error of a particular representation, it is a statement of equality that holds across all representations. It \emph{does not matter} which \(\phi\) you choose. Every random features representation has the same sample error as every other.

\bibliography{AllRandomFeaturesRepresentationsAreEquivalent}

\end{document}